\documentclass{article} 
\usepackage{graphicx}
\usepackage{wrapfig}
\usepackage{caption}
\usepackage{multirow}
\usepackage{amsmath,amsfonts,bm}
\newcommand{\R}{\mathbb{R}}
\newcommand{\argmax}{\mathrm{argmax}}
\title{\bf 
A precortical module for robust CNNs to light variations}
\author{R. Fioresi,  J. Petkovic}

\usepackage{amsthm}
\usepackage{amssymb}
\usepackage{color}

\newtheorem{proposition}{Proposition}[section]

\newtheorem{definition}[proposition]{Definition}

\newcommand{\cE}{\mathcal{E}}
\newcommand{\cK}{\mathcal{K}}

\newcommand{\cR}{{\mathcal{R}}}

\newcommand{\cS}{{\mathcal{S}}}

\newcommand{\regu}{\mathrm{reg}}
\newcommand{\beq}{\begin{equation}}
\newcommand{\eeq}{\end{equation}}
\newcommand{\lra}{\longrightarrow}

\newcommand{\ret}{E}
\newcommand{\vis}{V}
\newcommand{\gan}{\widetilde{E}}
\newcommand{\Rprime}{\widetilde{\cR}}
\newcommand{\fun}{F}
\newcommand{\dom}{D}

\newcommand{\reti}{RetiLeNet }
\newcommand{\lenet}{LeNet 5 }

\begin{document}

\centerline{\large \bf A precortical module for robust CNNs}

\medskip
\centerline{\large \bf  to light variations}

\bigskip
\centerline{R. Fioresi\footnote{University of Bologna, Italy},  
J. Petkovic\footnote{University of Mainz, Germany}}

\begin{abstract} 
We present a simple mathematical model for the mammalian low visual pathway, 
taking into account its key elements: retina, lateral
geniculate nucleus (LGN), primary
visual cortex (V1).
The analogies between the cortical level of the visual system 
and the 
structure of popular CNNs, used in image classification tasks, suggests  
the introduction of an additional preliminary convolutional module inspired 
to  precortical neuronal circuits to improve robustness with respect 
to global light intensity and contrast variations in the input images. 
We validate our hypothesis on the popular databases MNIST, FashionMNIST and 
SVHN, obtaining significantly more robust CNNs with respect 
to these variations, once such extra module is added. 
\end{abstract}

\section{Introduction}\label{intro-sec}

The fascinating similarities between 
CNN architectures and the modeling of the mammalian low visual pathway
are a current active object of
investigation (see \cite{montobbio}, \cite{bertoni}, 
\cite{ecker} and refs. therein). 
Historically, the study of border and visual perception started around 1920's 
with the Gestalt psychology 
formalization (Pr\"agnanz laws) of the
perception (lines, colors and contours, see \cite{mather} and refs. therein).
Then, the foundational work \cite{hubel} introduced a more
scientific approach to the subject, defining the concept of
receptive field, simple and complex cells, together with an anatomically
sound description of the visual cortex in mammals.

This study paved the way to the mathematical modeling for 
such structures. In particular special attention was given to
the primary visual cortex V1 (see \cite{hoffmann}),
whose orientation sensitive simple cells, together with the complex and
hypercomplex cells, inspired the modeling of algorithms
\cite{field2} contibuting to the spectacular success
of Deep Learning CNNs \cite{lbbh}.
With a more geometric approach in \cite{bc}, the introduction on V1 of a
natural subriemannian metric \cite{petitot, citti}, following
the seminal work \cite{mumford}, led to interesting
interpretations of the border completion mechanism.

While the mathematical modeling became more sophisticated 
(see \cite{petitot2} and refs therein), the insight into the 
physiological functioning of the visual pathway led to
more effective algorithms in image analysis \cite{fd, bertoni}.
In particular, the Cartan geometric point of view on the V1 modeling 
\cite{petitot2},
fueled new interest and suggests a physiological counterpart
for the new algorithms based on group equivariance in 
geometric deep learning approaches
(see \cite{bekkers, welling} and refs therein).

The purpose of our paper is to provide a simple mathematical model for
the low visual pathway, comprehending the retina, the lateral
geniculate nucleus (LGN) and the primary visual cortex (V1) and
to use such model to construct a 
preliminary convolutional module, that
we call \textit{precortical} to enhance the robustness
of popular CNNs for image classification tasks.
We want the CNN to gain the outstanding ability of the human eye
to react to large variations in global light intensity and
contrast. This can be achieved by mimicking the 
\textit{gain tuning} effect implemented the precortical portion of the mammalian visual pathway. This effect consists in the following:
since the low visual circuit needs to respond to a vast
range of light stimuli, that spans over $10$ orders of
magnitude, it is equipped with a lateral inhibition mechanism
which allows a low latency and a high sensitivity response.
This mechanism is functionally embedded in the center-surround receptive fields of retinal bipolar cells, retinal ganglions and LGN neurons and is able to achieve both border enhancing and decorrelation between the perceived light intensity in single pixels and the mean light value in a given image \cite{kandel}. 
We will show that,
if in the early levels of a CNN such receptive fields are learned in the form of convolutional filters,
there is a considerable improvement in accuracy, when
considering dimmed light or low contrast examples, 
i.e. examples not belonging to the statistic of the training set.

\medskip
The organization of this paper is as follows.
We start with a simple mathematical formalization of the low visual pathway, 
which accounts for its key components.
Though this material is not novel, we believe our terse presentation can 
greatly help to elucidate the relation between mathematical entities, 
like bundles or vector fields and the local physiological structure of 
retina and V1, together with their functioning.

The similarities between the inner structure of
CNNs and the physiological visual perception mechanism, 
once appropriately mathematically modeled as above, show that 
popular CNNs structure, for image classification tasks,
do not take into full consideration
the role of the precortical structures, which are responsible
for a correct adjustment to global light intensity and contrast in an image.
Our observations then, suggest the introduction of a 
precortical module
(see also \cite{montobbio, bertoni}), 
which mimics the functioning of retina and LGN
cells and reacts appropriately to the variations of the light
intensity and the contrast. Once such module is introduced, we
verify in our experiments on the MNIST, FashionMNIST and SVHM 
databases that our CNN shows robustness with respect to such large variations.
So effectively our experiments and subsequent
implementations show the importance of the module
introduced in \cite{bertoni}.
 
\medskip
The impact and potential of our approach is that a new simplified,
but accurate, mathematical modeling of the low visual pathway, can lead
to key cues on algorithm design, which add robustness and allow
high performances beyond the type of data the algorithm is trained with,
exactly as it happens for the human visual perception.
In a forthcoming research, we plan to further explore 
this study by analyzing the autoencoder
performances in the border completion, using the mathematical description 
via subriemannian metric geodesic. 

\section{Related work}\label{rw-sec}
The structure of the mammalian visual pathway 
was extensively explored in 
the last century both from an anatomical  and a functional point 
of view (see \cite{visual_anatomy} and
\cite{kandel} and refs therein). 
New aspects, however, are still 
discovered nowadays at each level of its structure: from new cellular types 
\cite{another_small_bistratified} to new functional circuits 
\cite{new_vp_circuit},
shedding more light  
on the formal process of visual information encoding 
(see \cite{direction_connectivity}, 
\cite{laminar_subnetwork}, \cite{intermosaic}). 
The striking correspondence between the training of mammals visual system
and popular CNN training was elucidated in \cite{as}, 
suggesting more exploration
into this direction is necessary. 
Similarities between deep CNN structure 
and the human visual pathway have been then increasingly 
explored in this framework.
From the first appearance of the \emph{neocognitron} \cite{fukushima}, 
the feature 
extracting action of the convolutional module has been widely 
adopted to tackle a 
number of visual tasks, developing models which presented 
striking analogies 
with mammalian cellular subtypes and their receptive fields 
\cite{lbbh, kri, bertoni, montobbio}. 
These similarities were studied, in particular, in relation 
to the deeper components of the visual pathway, starting from the hypercolumnar 
structure of the primary visual cortex V1 and continuing with 
superior processing 
centers, like V4 and the inferior temporal cortex 
(\cite{Mohsenzadeh} and the references therein). 
The exploration of  
similarities with the first stages of visual perception and in particular 
to the receptive field structure of retinal and geniculate units or with 
their precise regularizing effect on the perceived image appears also
in \cite{enroth1966}, \cite{dapello2020}, \cite{evans2021}, 
\cite{Zoumpourlis2017}.
Another recent attempt in this direction is also in \cite{bertoni}, where
a biologically inspired CNN is studied in connection with the Retinex
model \cite{land}, elucidated mathematically in \cite{provenzi} 
and then implemented \cite{morel} (see also refs. therein).
Also in \cite{morgenstern2014} and more recently in \cite{Strisciuglio2020},
\cite{babaiee2021}
appear some observations 
the behaviour of artifician neurons on contrast and more
generally robustness.

\section{The visual pathway}\label{vp-sec}

The visual pathway consists schematically of the following structures: 
the eye, the optic nerve, the lateral geniculates nucleus (LGN), the
optic radiation and the primary visual cortex (V1). In our discussion, we focus only 
on the retina, 
LGN and V1, because 
these are the parts in which the processing of the signal (i.e. the images) requires
a more articulate modeling.

The retina consists of photoreceptor cells, called \textit{receptors}, 
which measure the intensity 
of light. We model part of the retina, the {\it hemiretinal receptoral
layer},  
with a compact simply connected domain $\ret$ in $\R^2$.
We take a portion and not the whole retina, because
we want to avoid to cross the {\it median cleft line} in the visual field,
which is interpreted and processed in a more complex way (see
\cite{ah}).

We define then a \textit{receptorial activation field} as
a function
\beq\label{hem-rec}
\begin{array}{cccc}
\cR: & \ret & \lra & \R \\
 &(x,y)& \mapsto & \cR(x,y)
\end{array}
\eeq
associating to each point $(x,y)$, corresponding to a
receptor in the retina, an activation rate given by the
scalar $\cR(x, y)$. We do not assume $\cR$ to be continuous,
however, by its physiological significance, it will be bounded.

\begin{wrapfigure}{r}{0.49\textwidth}
  \begin{center}
    \includegraphics[width=0.49\textwidth]{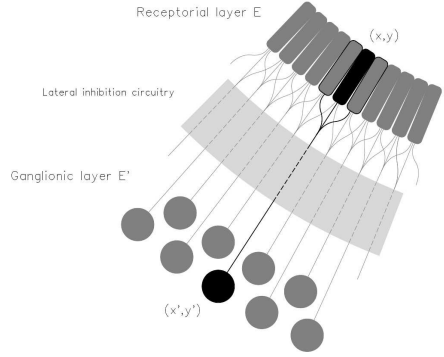}
  \end{center}
  \caption{Schematic structure of receptorial and ganglionic layers in retina}
\end{wrapfigure}

We may then 
interpret $\cR$ as coming from an image in the visual field. 
We will not distinguish between on and off receptors, as their
final effect on the downstream neurons is the same from a logical point of view.
We assume $\cR(x, y)$ to be proportional to the intensity of the light falling
on $(x, y)$. The ganglionic layer $\gan$ (see Fig. 1) 
sits a few layers down the visual pathway. 


To each receptor $(x, y)$ in $\ret$ there is a corresponding ganglion 
$(x',y')$ in $\gan$ with its
receptive field centered in $(x', y')$ so that we have a natural 
distance preserving identification between
$\ret$ and $\gan$, given by an isometry $G:\ret \lra \gan$.

\bigskip
The ganglionic activation pattern, however, 
is quite different from the hemiretinal receptors one (\ref{hem-rec}):
\begin{equation}\label{gan-rec}
\Rprime:\gan \lra \R, \quad
\Rprime(x',y') = \int_{U_{\rho}(x,y)}\cK(u,v)\cR(u,v)\,du\,dv
\quad            \hbox{with} \quad G(x,y)=(x',y') 
\end{equation}
where
            \begin{equation*}
                U_{\rho}(x,y)=\big\{(u,v) \in \mathbb{R}^2 : 
(u-x)^2+(v-y)^2 \leq \rho^2\big\},
\end{equation*}
 \begin{equation*}
                \cK(x,y)=\begin{cases} \pm 1 & \mathrm{if} \quad
(u-x)^2+(v-y)^2 \leq (\rho-\epsilon)^2 \\ \\
\mp 1 & \mathrm{if} \quad
(\rho-\epsilon)^2 < (u-x)^2+(v-y)^2 \leq \rho^2 
\end{cases}
\end{equation*}
The identification  between $\ret$ and $\gan$ given by $G$ is {\sl not}
a manifold morphism: in fact the correspondence between functions
follows (\ref{gan-rec}), which is 
an integral transform with kernel $\cK(x,y)$. This models 
effectively the mechanism of
firing of hemiretinal receptors: for each activation disc
we always have an inhibition crown around it. This is
the key mechanism, responsible for the border and contrast
enhancement,
that we shall implement with our
precortical module in Sec. \ref{exp-sec}.

We notice an important consequence, whose proof is  in App. \ref{app-sec1},
to ease the reading.

\begin{proposition}\label{lipschitz-prop}
The hemiretinal ganglionic activation field $\Rprime$ is Lipschitz continuous 
in both variables on $\gan$.
\end{proposition}

This proposition encodes the fact that the visual system
is able to reconstruct a border percept also for non continuous
images. This fact was already noticed in \cite{mumford}, where edge 
interruptions are taken into account and
defined as cusps or elementary catastrophes. 
We illustrate this phenomenon in Fig. 2, which is perceived as a 
unitary curvilinear
shape with clear borders, though composed of isolated black dots.
There is indeed a further smoothing reconstruction carried out
in the LGN, so that our image will provide us with a smooth function
in place of $\Rprime$, \cite{origins-neurogeo}.
We shall not describe such modeling for the present work.

\medskip
We now come to the last portion of the visual pathway: the 
{\it primary visual hemicortex}, that we shall still denote with $V1$. 
The \textit{retinotopic map} is a distance preserving
homeomorphism between the hemiretinal receptorial layer and 
the primary visual hemicortex (see \cite{ah} for a concrete realization).
We can therefore identify $V1$ with a compact domain $\vis$ in $\R^2$.

\setlength{\intextsep}{1mm}
\begin{wrapfigure}[14]{r}{0.35\textwidth}
    \centering
    \includegraphics[width=0.3\textwidth]{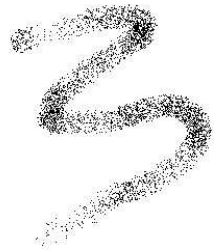}
    \captionsetup{width=0.3\textwidth}
  \caption{Border percepts for a non continuous image}
\end{wrapfigure}

The pointwise identification, however, is not sufficient to model
the behaviour of $V1$. In fact for each point of the visual cortex 
we have three main information to take into account:
\begin{enumerate}
    \item The absolute position of the corresponding point 
in the retina, which receives a stimulus;
    \item the orientation $\theta$ of some perceived edge through the point
(simple and complex cells);
    \item the curvature $k$ of some perceived edge through 
the point (hypercomplex cells).
\end{enumerate}
While the first feature is encoded by the points in the domain $V$,
additional modelization is needed for
the construction of orientation and curvature information. 
For this reason, a point in $V1$ is
identified with a biological structure, called \textit{hypercolumn}, which
contains a plethora of cell families, 
each analyzing a specific aspect of the visual information.
We assume that at each point of $V1$ is present a
full set of simple, complex and hypercomplex orientation columns 
(a so called  \textit{full orientation hypercolumn}).

We briefly
recap the various types of cells in $V1$, since their behaviour plays a key role in our modeling. 
\begin{itemize}
\item {\it Simple cells}: 
they exhibit a rich and multifaceted behavioural pattern in response to positional, 
orientational and dynamic features of a light input. 
Traditionally, for their modeling is used
the asymmetrical part of a Gabor filter,
though in recent works more importance is given
to the role of the LGN and the actual 
neuroanatomical structure (see \cite{lindeberg, petkov}).
\item {\it Complex cells}: they
receive input directly from simple cells 
(see \cite{hubel}) and they show a 
linear response depending on the orientation of
some static stimulus over a certain receptive field. It is 
important to note that
this response is invariant under a $180^0$ rotation of the stimulus.
\item {\it Hypercomplex cells}: 
they are also called end-stopped cells and they are maximally activated by
oriented stimuli positioned in the central region of their receptive field. They
are maximally inhibited by peripheric stimuli with the same orientation. Such
cells are therefore not particularly reactive to long straight lines, 
firing briskly if perceiving curves or corners.
\end{itemize}
As we shall see in our
next section, in order to take into account all different cells behaviours
in the hypercolumn, we need to model $V1$ using the mathematical concept
of {\sl fiber bundle}.

\section{The Primary Visual Cortex}\label{v1-sec}

We want to provide a mathematical description of V1, 
starting from the actual neuroanatomical structures and taking into account
the combined effects of complex and simple cells.
Though this material is mostly known (see \cite{petitot, citti, bc, hoffmann}),
our novel and simplified presentation will elucidate the role of the key components
of the visual pathway, while keeping a faithful representation of the 
neuroanatomical structures.

Let $\dom \subset \R^2$ be compact, simply connected.
We start by defining the {\sl orientation} of a function $F:\dom \lra \R$;
it will be instrumental for our modelization of $V1$.
On the manifold $D \times S^1$ 
define the vector field $Z$:
$$
Z(x,y,\theta) = -\sin \theta \; \partial_x+\cos \theta \; \partial_y
$$
where as usual $\partial_x$, $\partial_y$ form a basis for the tangent space
of $\R^2$ at each point and $\theta$ is the coordinate for $S^1$.
To ease the notation we drop $x,y$ in the expression of $Z$, writing
$Z(\theta)$ in place of $Z(x,y,\theta)$.

\begin{definition}\label{def-or}
{\rm Let $\dom \subset \R^2$ be compact and simply connected
and $\fun:\dom \lra \R$ a smooth function. 
Let $\regu(\dom) \subset \dom$ be the subset of regular points of 
$\fun$ (i.e. $p \in D$, $dF(p) \neq (0,0)$). We define the \textit{orientation} of $\fun$ as: 
\beq    
\begin{array}{ccccc}
        \Theta: & \regu(E') &\longrightarrow & S^1&\\
        & (x,y) &\longmapsto &\Theta(x,y) &:= \argmax_{\theta \in S^1}
{\bigl\{Z(\theta)\fun(x,y)\bigr\}}
    \end{array}
\eeq}
\end{definition}

The map $\Theta$ is well defined because of the following proposition. 

\begin{proposition}\label{border-prop}
Let $\fun:\dom \lra \R$ as above and $(x_0,y_0)\in \dom$ a regular point for $\fun$.
Then, we have the following:
\begin{enumerate}
\item There there exists a unique $\theta_{x_0,y_0} \in S^1$ for which
the function $\zeta_{x,y}: S^1 \longrightarrow \R$, 
$\zeta_{x,y}(\theta):= Z(\theta) \, \fun(x,y)$ attains its maximum.
\item The map $\Theta:\regu(D) \longrightarrow  S^1$,
$\Theta(x,y)=\theta_{x,y}$
is well defined and differentiable.
\item The set:
$$
\Phi=\{(x,y,\Theta(x,y)) \in \dom \times S^1 : \Theta(x,y) = \theta_{x,y}\}
$$
is a regular submanifold of $\dom \times S^1$.
\end{enumerate}
\end{proposition}

\begin{proof}
(1). Since $\zeta_{x,y}$ is a differentiable function on a compact domain
it admits maximum, we need to show it is unique. We can explicitly
express:
$$
\zeta_{x,y}(\theta)=-\sin \theta \; \partial_xF+\cos \theta \; \partial_yF
$$
Since $(\partial_xF, \partial_yF)\neq (0,0)$ and it is constant,
by elementary considerations, taking the derivative of $\zeta_{x,y}$ with
respect to $\theta$ we see the maximum is unique.

(2). $\Theta$ is well defined by $(1)$ and differentiable.

(3). It is an immediate consequence of the implicit function theorem.

\end{proof}

Notice the following important facts:
\begin{itemize}
    \item the locality of the operator $Z(\theta)$  
mirrors the locality of the hypercolumnar anatomical connections;
    \item  its operating principle
is a good description of the combined action of simple and complex cells
(though different from their individual behaviour);
\end{itemize}


We now look at an example given by Fig. \ref{fig:direction_derivative} 
of the behaviour of $Z$ on an image, that we view as a function
$F:\dom \lra \R$.

Here $\dom \cong \gan$, that is $F$ is a hemiretinal
ganglionic receptive field.

We can see in our example that for each 
point $(x,y)$ near the border of the dark circle represented
by the image $F$, there exists a value 
$\theta(x,y)$ 
for which the function $Z(\theta) F(x,y)$ is maximal
(we color in white the maximum, in black the minimum). 
This value is indeed 
the angle between the tangent line to a visually perceived 
orientation in $(x,y)$ 
and the $x$ axis.

\begin{figure}[!b]
\hspace{-1cm}
    \begin{tabular}{ccccc}
        \includegraphics[width=.2\textwidth]{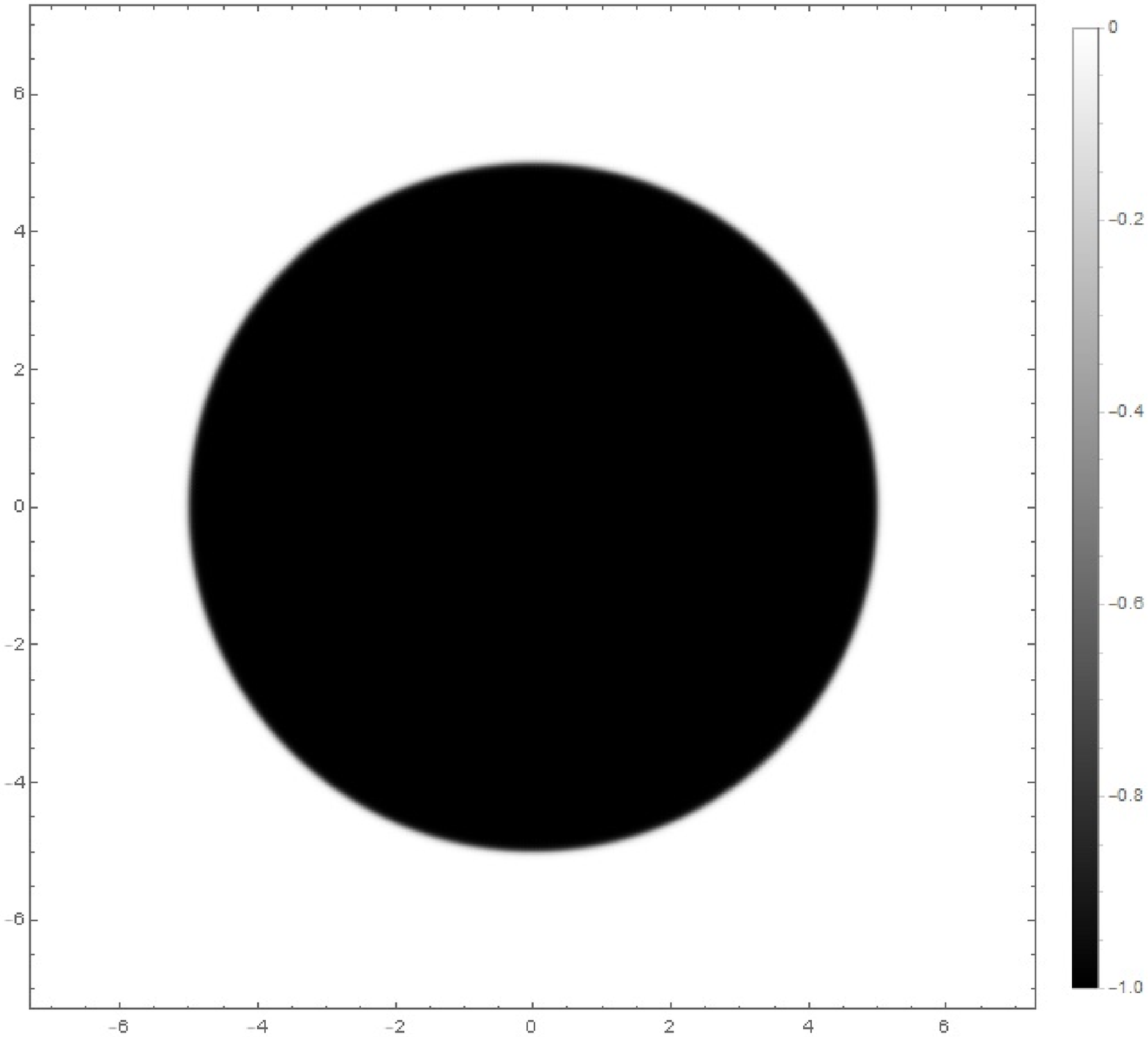}
        \includegraphics[width=.2\textwidth]{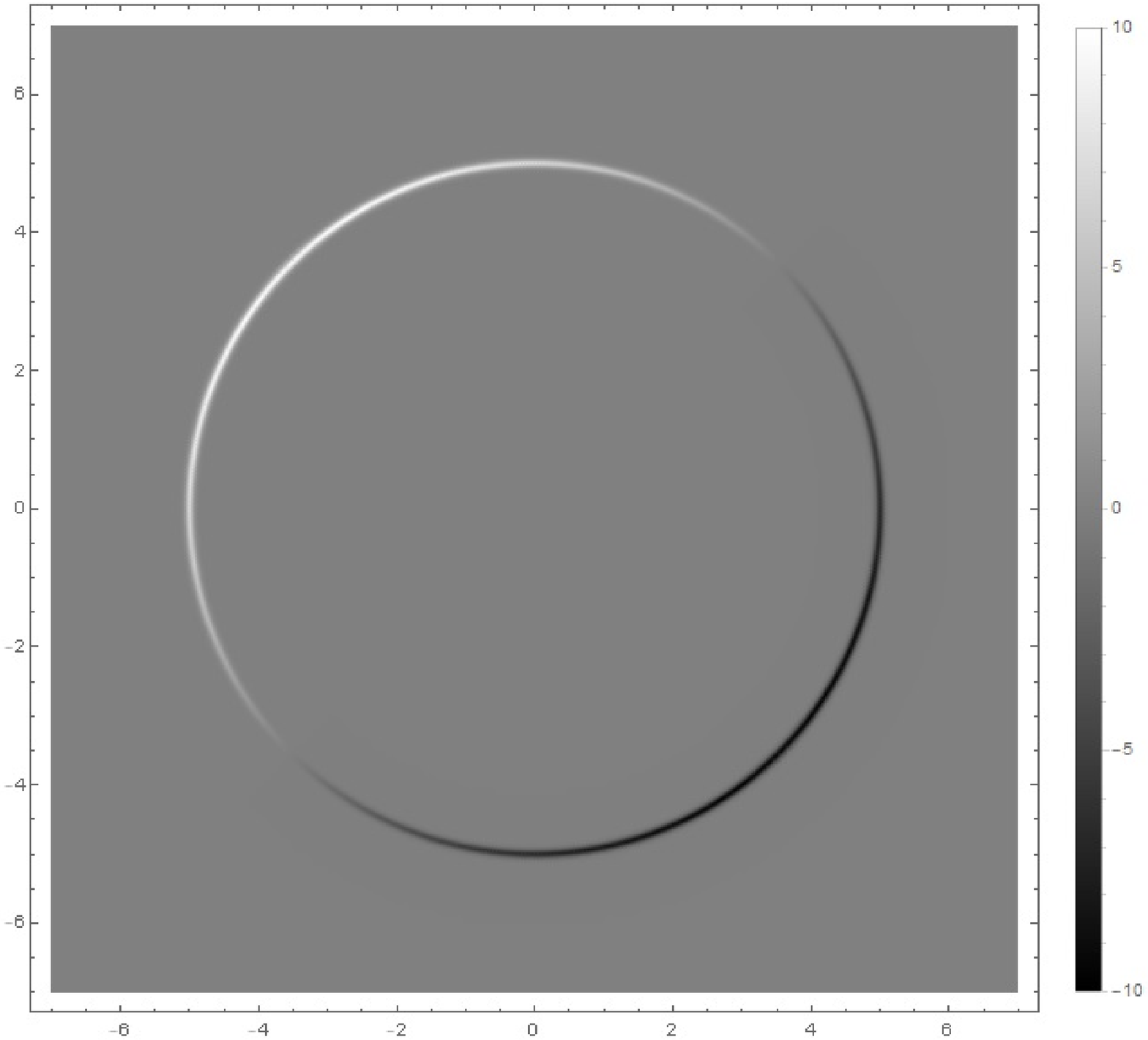}&
        \includegraphics[width=.2\textwidth]{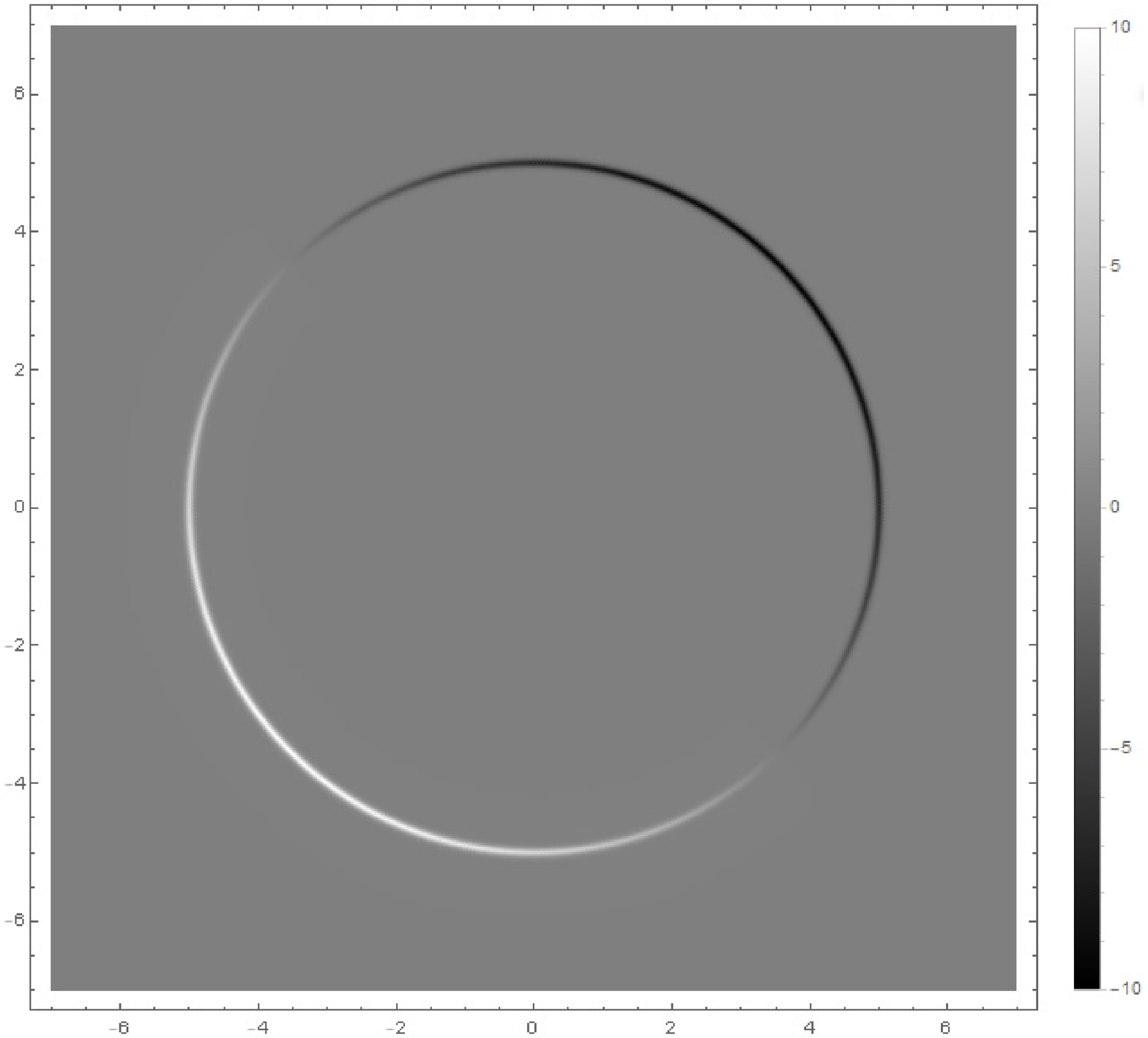}&
        \includegraphics[width=.2\textwidth]{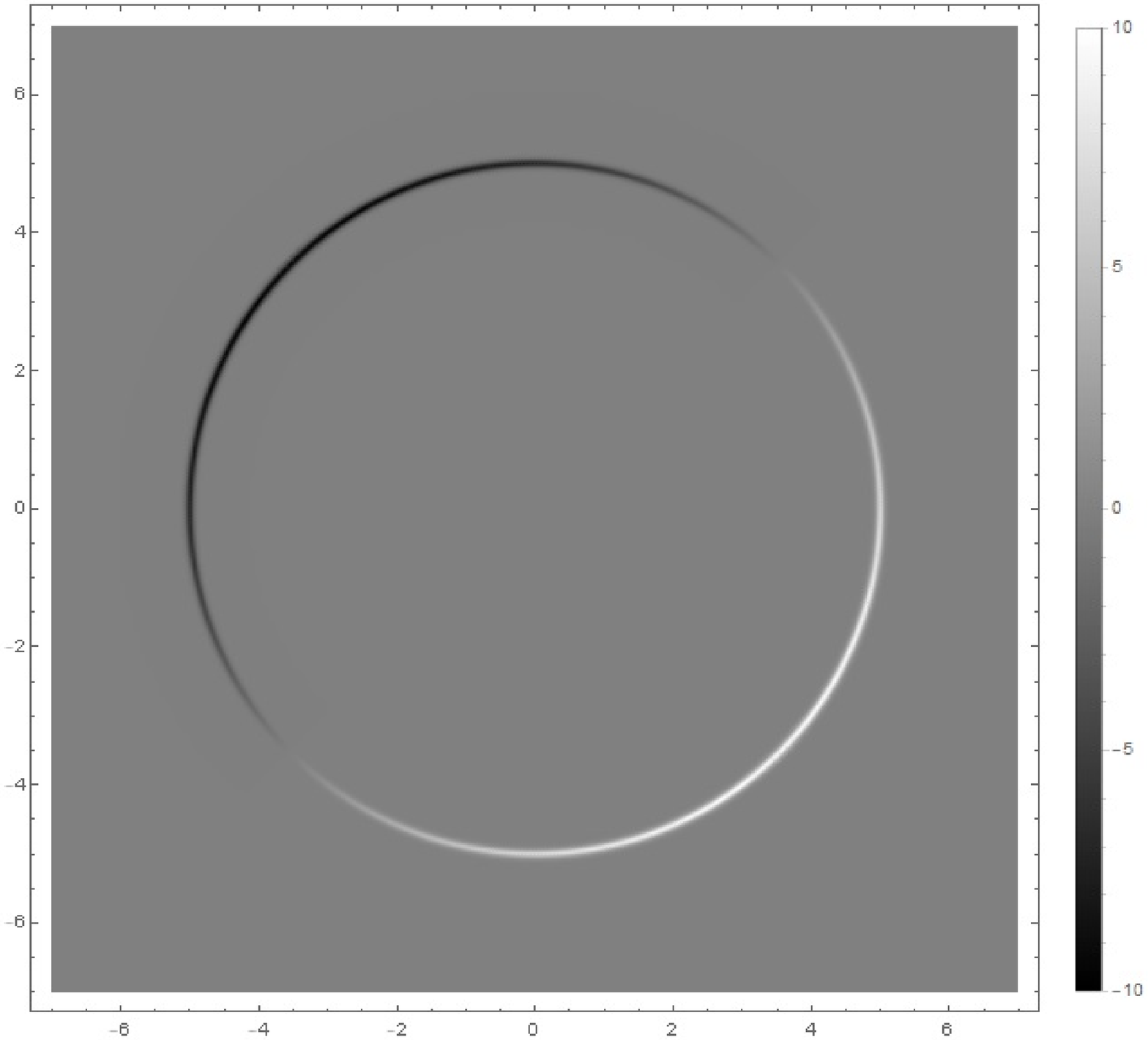}&
        \includegraphics[width=.2\textwidth]{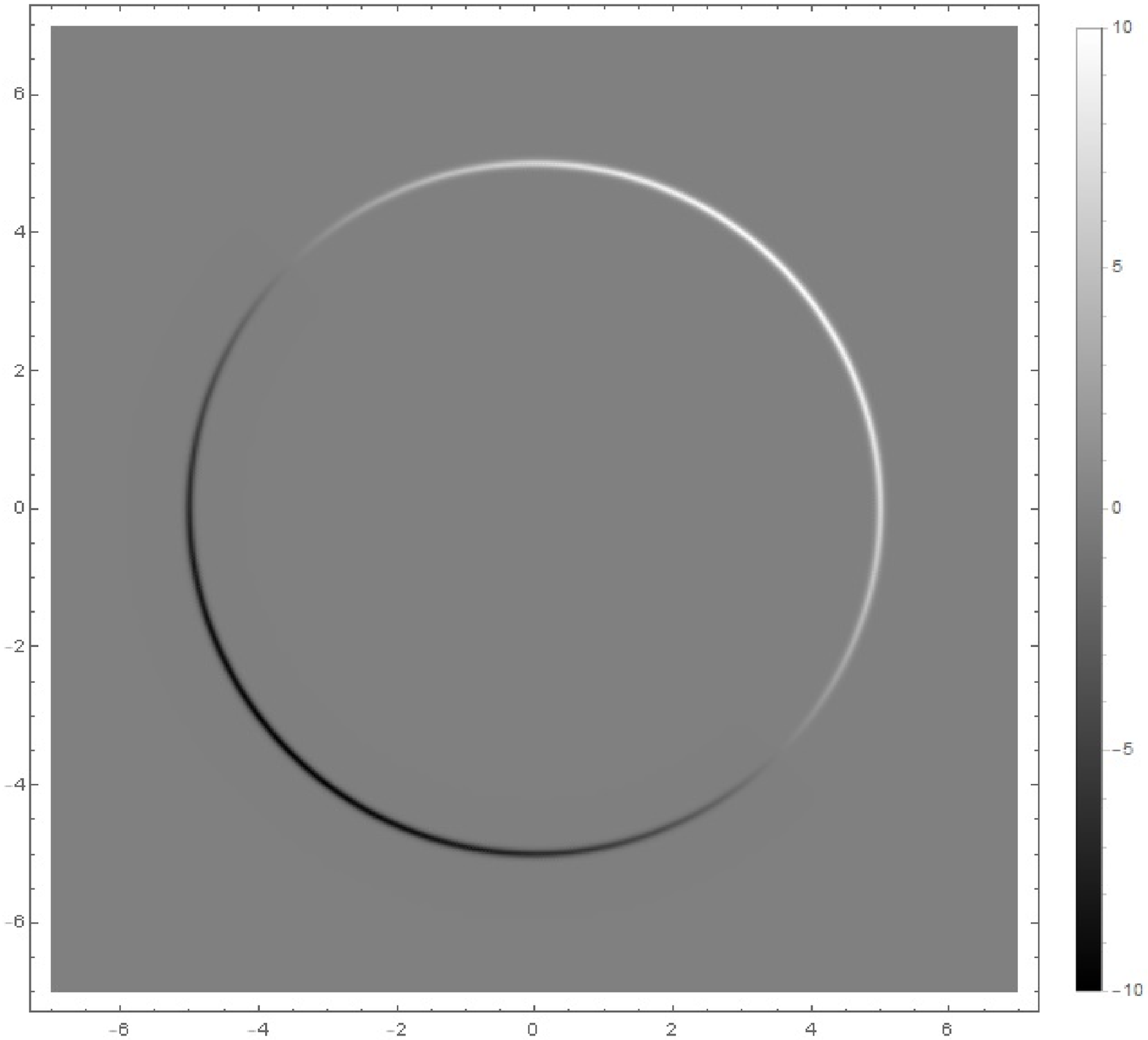}
    \end{tabular}
\caption{
Graphs of $F(x,y)=-\frac{1}{1+(0.04 \,x^2+0.04 \, y^2)^{100}}$, $Z(\frac{\pi}{4})F$, $Z(\frac{3\pi}{4})F$, 
$Z(\frac{7\pi}{4})F$ and $Z(\frac{5\pi}{4})F$.}
    \label{fig:direction_derivative}
\end{figure}

Since in $V1$ at each point we need to consider all possible orientations, coming from different receptive fields,
following \cite{petitot, citti, hoffmann}, we model $V1$ as the
fiber bundle $\cE:= \vis \times S^1 \lra \vis$, that we
call the \textit{orientation bundle}. Notice that $\vis$ is a domain
$\dom \subset \R^2$ and naturally identified with $\gan$ and $\ret$ 
with a distance preserving homeomorphism.
Then, each ganglionic receptive field $\Rprime$, after
LNG smoothing will give us a smooth section $\Theta$ of $\cE$, through
the identification we detailed above and elucidated in our example.

\section{Experiments}\label{exp-sec}

We now take into exam a popular CNN: \lenet and, based on
the analogies with the low visual pathway elucidated in 
our previous discussion, we modify it by introducing
a {\sl precortical} module.

\begin{wrapfigure}{r}{.25\textwidth}
    \includegraphics[width=.25\textwidth]{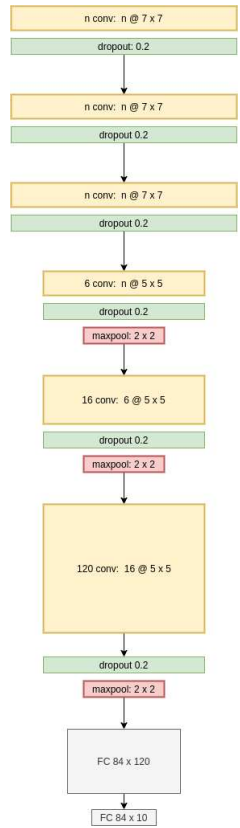}
    \caption{RetiLeNet structure; n = image color channels}
    \label{fig:model}
\end{wrapfigure}

The \lenet convolutional layers
have surprisingly strong similarities 
with the cortical directional hypercolumns: 
the stacked organization of the layers closely resembles
the feed forward structure of the simple - complex - 
hypercomplex cell path (see Sec. \ref{v1-sec}).
Also, the weights of the first convolutional layer converge 
to an integral kernel (see eq. (\ref{gan-rec})), which makes
it sensitive to directional feature extraction.
To allow the emergence of the precortical {\sl gain tuning effect}
(see Sec. \ref{intro-sec}), we structured our model as follows. 
We insert at the beginning the precortical module
consisting of a sequence of 
convolutional, dropout and hyperbolic tangent 
activation layer, repeated three times.
This number was chosen according to \cite{visual_anatomy} and corresponds to the three vertical types of cells present in the precortical portion of the visual path (bipolar cell, retinal ganglion, 
LGN neuron). The number of features in each convolutional layer was chosen equal to the number of color channels (one for MNIST and FashionMNIST, three for SVHN) and the padding was set to \verb+(kernel_size-1)/2+ to preserve the 
dimension of the input data. The output of this 
module, denominated \textit{precortical module} was then fed to a slightly modified version of \lenet (Fig. \ref{fig:model}). The number 
\verb+kernel_size+
is taken as an hyperparameter and chosen for each dataset,
according to the best performance.
Finally, a modified \lenet without the precortical module was trained on the same datasets and used for accuracy comparison. The training
for both CNNs, that we denote \reti and \lenet (with and without the precortical module),
was carried out without any image preprocessing or data augmentation using the ADAM optimizer, with learning rate $0.001$ and a batch size of 128.\\
Since we want to study the effect on robustness of the precortical module, we tested the two models on transformed test sets. In particular, we considered the two transformations:
\begin{enumerate}
    \item {\bf Mean offset (luminosity change)}: each pixel $x_i$ of the input image $X$ was shifted by an offset value $\mu$ as
    \begin{equation*}
        x_i \leftarrow x_i - \mu
    \end{equation*} From a phenomenological point of view this represents a variation in the global light intensity.
    \item {\bf Distribution scaling (contrast change)}: 
    each pixel $x_i$ of a given image is modified with a deviation parameter $\sigma$ as 
    \begin{equation*}
        x_i \leftarrow \frac{x_i - \bar{X}}{\sigma} + \bar{X}
    \end{equation*}
    with $\bar{X}$ the mean pixel value in the image. In this fashion we obtain a modification in the image contrast.
\end{enumerate}
\begin{figure}[h!]
    \hspace{-6mm}
    \begin{tabular}{c p{2mm}c} 
        \includegraphics[width=.48\textwidth]{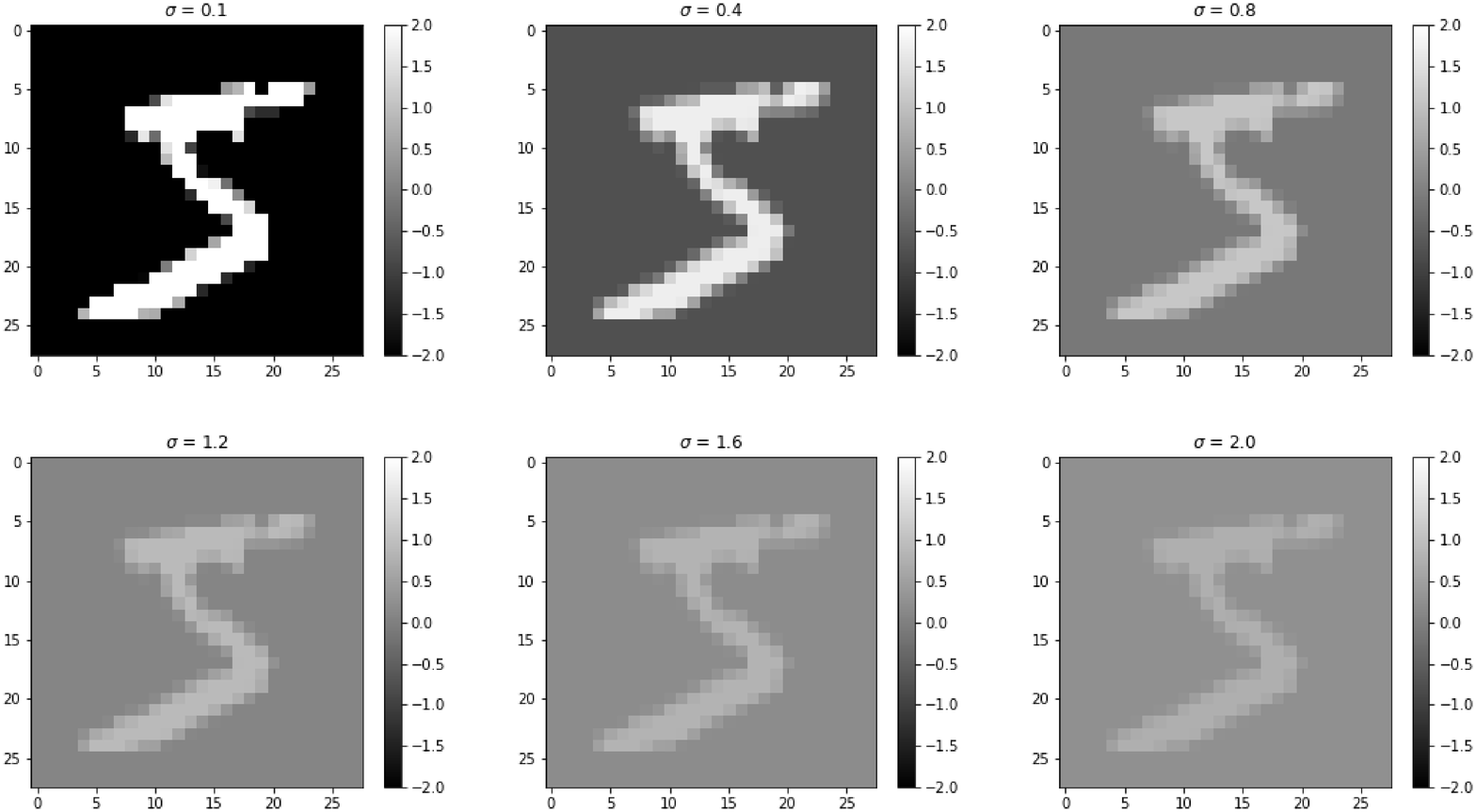}
        &&
        \includegraphics[width=.48\textwidth]{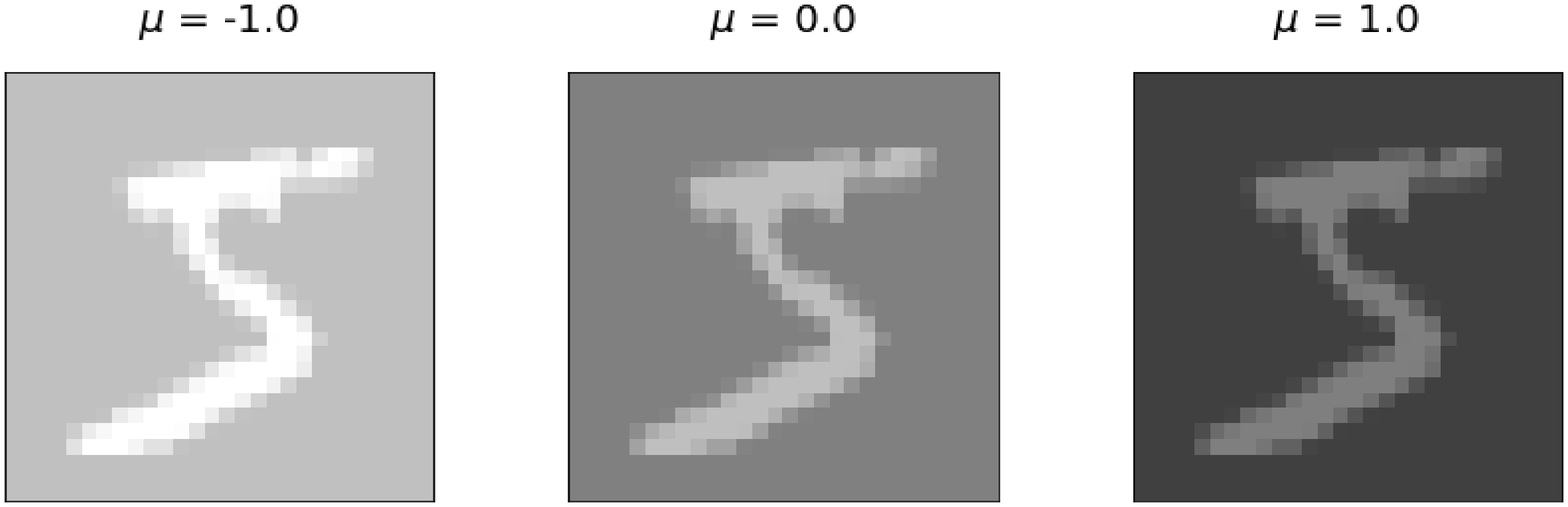}
    \end{tabular}
    \caption{
    New samples generated by contrast reduction via $\sigma$ (left) and light dimming via $\mu$ (right).} 
\end{figure}
\medskip

\subsection{Results}

After training both models \lenet and \reti (\lenet with
the precortical module) on MNIST, FashionMNIST and SVHN, 
choosing the hyperparameter \verb+kernel_size+ equal to
7 for RetiLeNet, we observe that in general, \reti performed better than \lenet on the three datasets, as shown in Table \ref{tab:accs}. We also notice that the addition of the precortical module did not impact greatly on the final accuracy for the unmodified datasets, that is taking $\mu=0$, $\sigma=1$, accounting for a slight score improvement on SVHN and, surprisingly, a slight worsening on MNIST and FashionMNIST (see \ref{tab:accs}).\\
\begin{table}
 \begin{center}
    \centering
    \begin{tabular}{r | c c c | c c c}
        \hline
         & \multicolumn{3}{c}{$\mu$} & \multicolumn{3}{c}{$\sigma$}\\
        \multicolumn{1}{l|}{\textbf{MNIST}}& -2.0 & 0 & 1.0 & 0.1 & 1.0 & 3.9 \\
         \hline
        \lenet & 0.120 & 0.991 & 0.119& 0.915 & 0.991 & 0.485 \\
        \reti 5 & 0.097 & 0.988 & 0.787 & 0.984 & 0.988 & 0.907 \\
        \hline \multicolumn{7}{c}{} \\
        \hline
        \multicolumn{1}{c}{}& \multicolumn{3}{|c|}{$\mu$} & \multicolumn{3}{c}{$\sigma$}\\
        \multicolumn{1}{l|}{\textbf{FashionMNIST}}& -2.0 & 0 & 2.0 & 0.1 & 1.0 & 3.9 \\
         \hline
        Lenet 5 & 0.168 & 0.887 & 0.100 & 0.836 & 0.887 & 0.422 \\
        Mod Lenet 5 & 0.770 & 0.880 & 0.781 & 0.805 & 0.880 & 0.806 \\
        \hline \multicolumn{7}{c}{} \\
        \hline
        & \multicolumn{3}{c}{$\mu$} & \multicolumn{3}{c}{$\sigma$}\\
        \multicolumn{1}{l|}{\textbf{SVHN}} & -2.0 & 0 & 2.0 & 0.1 & 1.0 & 3.9 \\
         \hline
        Lenet 5 & 0.806 & 0.868 & 0.006 & 0.723 & 0.868 & 0.776 \\
        Mod Lenet 5 & 0.831 & 0.886 & 0.738 & 0.849 & 0.886 & 0.832 \\
        \hline
    \end{tabular}
    \captionsetup{width=.8\textwidth}
    \caption{Accuracies achieved with and without the precortical module versus $\mu$ and $\sigma$ sweeps. Notice: for MNIST the highest $\mu$ reported is 1.0 while it is 0.2 for the other two datasets. }
    \label{tab:accs}
\end{center}
\end{table}

We notice a better performance of
\reti, with respect to \lenet, on the transformed dataset, that
is when we modify the contrast (via $\sigma$) and the light intensity
(via $\mu$), see Fig. \ref{fig:acc1}, \ref{fig:acc2}, \ref{fig:acc3}.
The improvement in performance is particularly noticeable in the case of low contrast ($\sigma$ = 3.9) and very dark examples ($\mu$ = 2 for SVHN and FashionMNIST, $\mu$=1 for MNIST), where the 
accuracy of \reti is consistently greater than 75\% 
in comparison
with \lenet accuracies, which are far lower.
The reason for this behaviour is the strong stabilizing effect that the first (convolutional) layer of the precortical module has  
on the input image, 
in the analogy to the modeling for the receptive fields
introduced in (\ref{gan-rec}).

To elucidate this stabilizing effect,  we choose an example from each dataset and look at the corresponding first hidden output, while the parameters $\mu$ and $\sigma$ vary (see Fig. \ref{fig:mnist_box}, \ref{fig:fashionmnist_box}, \ref{fig:svhn_box}). 
As far as the $\mu$ shift is concerned we see that this effect
is very evident: for large variations of $\mu$ we have a pixel average close to zero after the first convolutional layer.
In particular, 
we notice a correlation between a worse stabilizing action 
and a worse model performance, by comparing what happens with
MNIST with respect to FashionMNIST and SVHN.
Despite the small number of datasets considered, we think it may
hint to a significant phenomenon worth a further investigation.

Moreover, we can see close similarities between this effect and the actual biological lateral inhibition phenomenon 
emerging at the bipolar cell level in the retina
corresponding indeed to the first of our precortical convolutional layers. 
This biological phenomenon is responsible for the gain adjustment and border enhancement mechanisms which allow the visual system to respond correctly to strong changes in ambient light conditions.

\begin{figure}[h!]
    \centering
    \includegraphics[width=.95\textwidth]{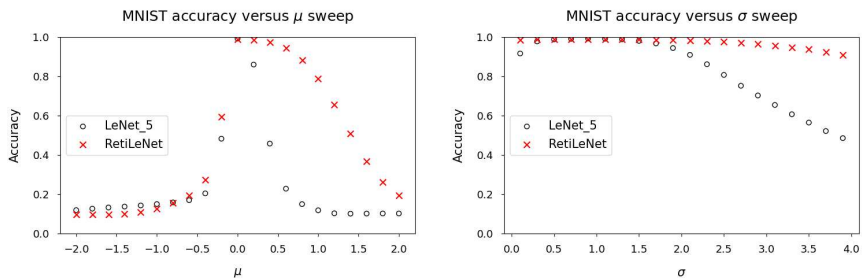}
    \caption{Accuracy for $\mu$, $\sigma$ variations in MNIST}
    \label{fig:acc1}
\end{figure}

\begin{figure}[h!]
    \centering
    \includegraphics[width=.95\textwidth]{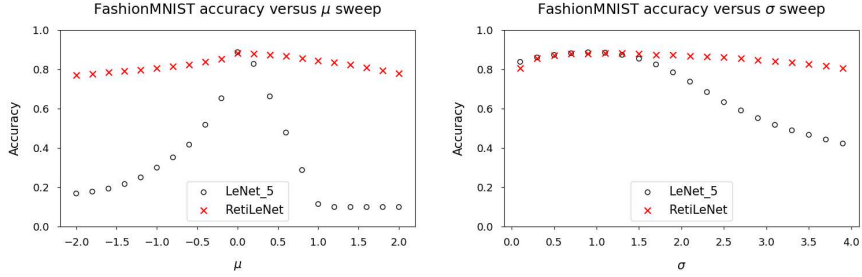}
    \caption{Accuracy for $\mu$, $\sigma$ variations in FashionMNIST}
    \label{fig:acc2}
\end{figure}

\begin{figure}[h!]
    \centering
    \includegraphics[width=.95\textwidth]{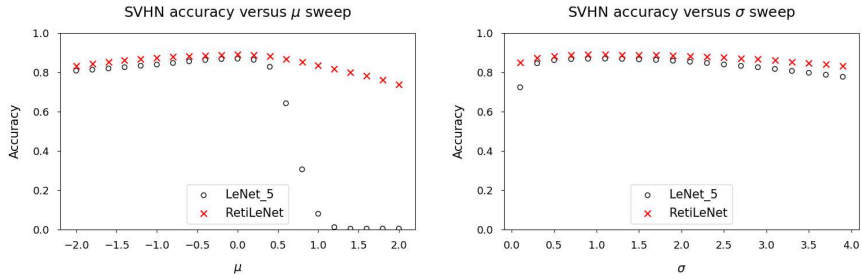}
    \caption{Accuracy for $\mu$, $\sigma$ variations in SVHN}
    \label{fig:acc3}
\end{figure}

\begin{figure}[!h]
    \centering
    \includegraphics[width=.95\textwidth]{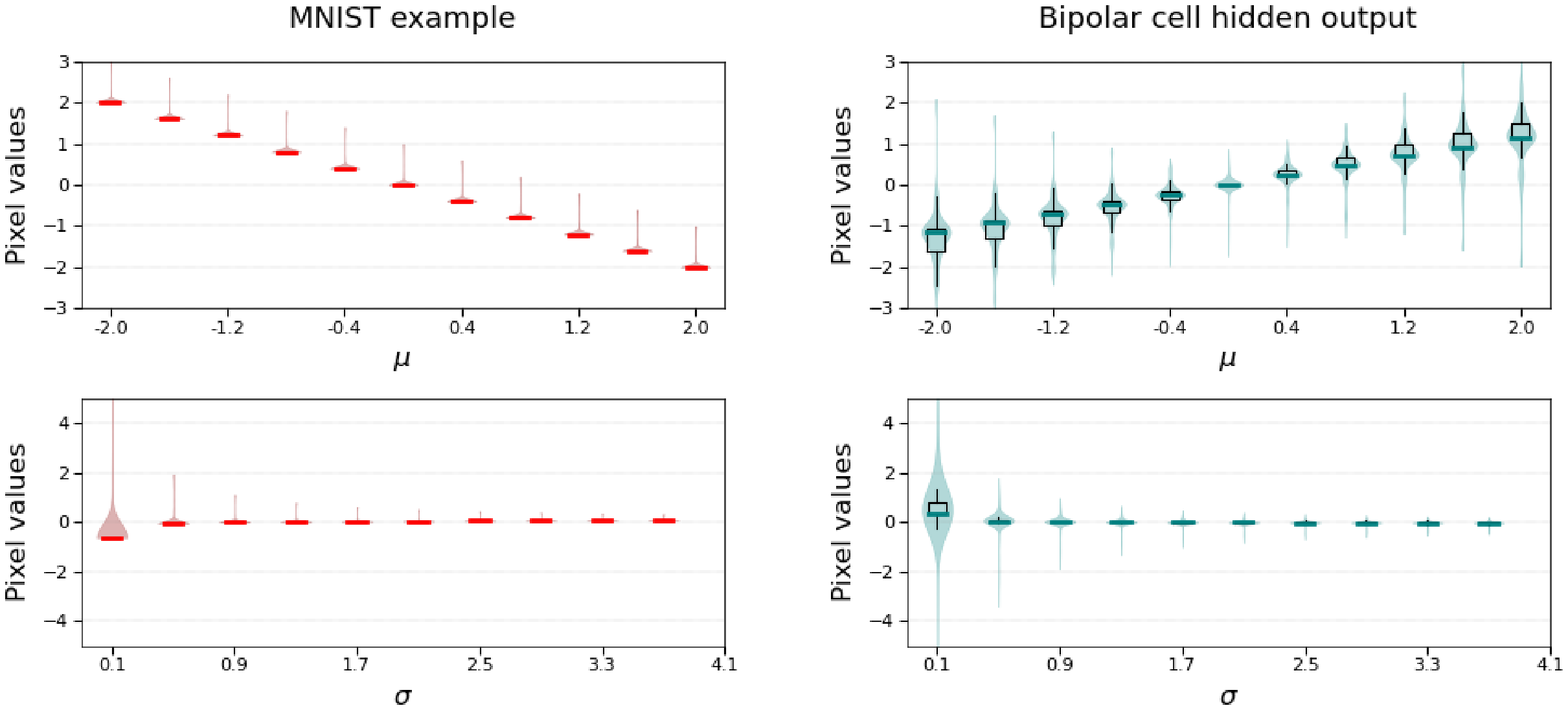}
    \caption{Pixel value distributions before (red) and after (green) first precortical convolutional layer for $\mu$, $\sigma$ variations in MNIST. (whiskers 1.5, violin kernel bandwidth 1.06)}
    \label{fig:mnist_box}
\end{figure}

\begin{figure}[h!]
    \centering
    \includegraphics[width=.95\textwidth]{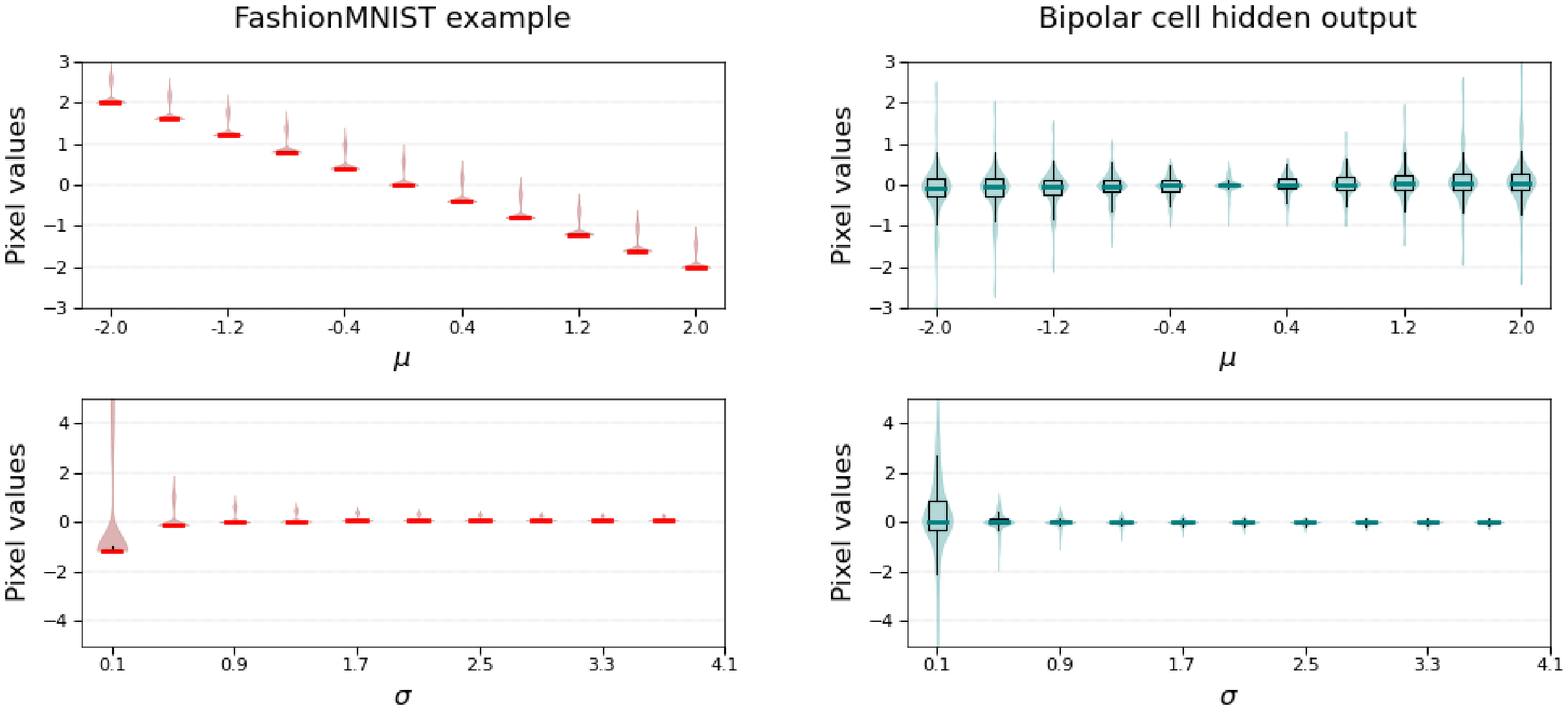}
    \caption{Pixel value distributions before (red) and after (green) first precortical convolutional layer for $\mu$, $\sigma$ variations in FashionMNIST. (whiskers 1.5, violin kernel bandwidth 1.06)}
    \label{fig:fashionmnist_box}
\end{figure}

\begin{figure}[h!]
    \centering
    \includegraphics[width=.95\textwidth]{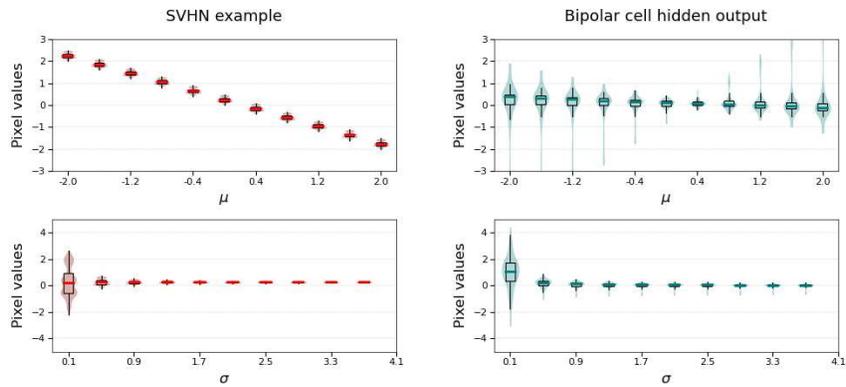}
    \caption{Pixel value distributions before (red) and after (green) first precortical convolutional layer for $\mu$, $\sigma$ variations in SVHN. (whiskers 1.5, violin kernel bandwidth 1.06)}
    \label{fig:svhn_box}
\end{figure}

\section{Conclusions} 
Our simple mathematical model of the low visual
pathway of mammals retains the descriptive accuracy   
of more complicated models and it is better suited to elucidate
the similarities between visual physiological structures
and CNNs. The precortical neuronal module,
inspired by our mathematical modeling, 
when added to a popular CNN (\lenet), gives
a CNN \reti that mimics the border
and contrast enhancing effect as well as the mean light decorrelation action of horizontal-bipolar cells, retinal ganglions and LGN neurons.
Hence
such addition, which performed
extremely well on datasets with large variations
of light intensity and contrasts, improves the CNN robustness with respect to 
such variations in generated input images, 
strongly improving its inferential power on data
not belonging to the training statistics (Fig. \ref{fig:acc1},
\ref{fig:acc2}, \ref{fig:acc3}).
We believe that this strong improvement is directly correlated with the stabilizing action performed by the first 
precortical convolutional layer, in complete analogy with the behaviour of bipolar cells in the retina.
We validate our hypotheses
obtaining our results on MNIST, FashionMNIST and SVHN datasets.


\section{Acknowledgements}
R. Fioresi wishes to thank Prof. Citti, 
Prof. Sarti and Prof. Breveglieri for helpful discussions and comments.

\newpage

\bibliographystyle{plain}
\bibliography{jr-iclr-lanl}

\appendix
\section{Proof of Proposition \ref{lipschitz-prop}}\label{app-sec1}

We first notice that, from a mathematical point of view,
we can  identify $E$ and $\gan$, 
hence we will state our result in this simplified, but 
mathematically equivalent fashion. 

In our notation, $D$ is the domain $E$ identified with $\gan$
and $\cS$ our receptive field $\cR$, while $S$ denotes $\Rprime$.
Notice that, given our physiological setting, both
$\cR$ and $\Rprime$ are bounded on $D$.

For notation and main definitions, see \cite{rudin}. 
\begin{proposition}
Let $D$ be a compact domain in $\R^2$. Consider the function:
\begin{equation}\label{gan-rec}
\cS:D \lra \R, \quad
\cS(x,y) = \int_{U_{\rho}(x,y)}\cK(u,v)S(u,v)\,du\,dv 
\end{equation}
where $S:D \lra \R$ is an arbitrary function, $S(D)$ is bounded and
(for a fixed $\rho$):      
\begin{equation*}
                U_{\rho}(x,y)=\big\{(u,v) \in \mathbb{R}^2 : 
(u-x)^2+(v-y)^2 \leq \rho^2\big\},
\end{equation*}
 \begin{equation*}
                \cK(x,y)=\begin{cases} \pm 1 & \mathrm{if} \quad
(u-x)^2+(v-y)^2 \leq (\rho-\epsilon)^2 \\ \\
\mp 1 & \mathrm{if} \quad
(\rho-\epsilon)^2 < (u-x)^2+(v-y)^2 \leq \rho^2 
\end{cases}
\end{equation*}

Then $\cS$ is Lipschitz continuous in both variables on $D$.
\end{proposition}

\begin{proof}
Since $S(D)$ is bounded, we have $N \leq S(x,y) \leq M$
for all $(x,y)\in E$.
We need to show
that there exists $L \in \mathbb{R}$ so that 
\begin{equation}\label{eq-lipsch}
        |\cS(p )-\cS(q )| < L\,\|p -q \|, 
\hspace{5mm} \quad \hbox{for all} \quad p=(x_p,y_p) ,\,q=(x_q,y_q) \in E
    \end{equation}

Let $A=U_{\rho}(x_p,y_p)$
and $B=U_{\rho}(x_q,y_q)$. Then: 
    \begin{align*}
        |\cS(p)-\cS(q)| &= 
\Big|\int_{A}S(u,v)\, du \, dv - \int_{B}S(u,v)\, du \, dv \Big|\\
        &= \Big|\int_{A\smallsetminus B}S(u,v)\, du \, dv - 
\int_{B\smallsetminus A}S(u,v)\, du \, dv \Big| \\
        &\leq \Big|\int_{A\smallsetminus B}M\, du \, dv -
        \int_{B\smallsetminus A}N\, du \, dv \Big|
    \end{align*}
    We shall now consider two cases:
    \begin{enumerate}
        \item $\|p-q\| > 2\rho$: 
this implies that $A \smallsetminus B = B \smallsetminus A = \varnothing$. In this case 
        \begin{equation*}
            \Big|\int_{A\smallsetminus B}M \, du \, dv - 
\int_{B\smallsetminus A}N\, du \, dv \Big| =  
\left(M-N\right)\;\pi \rho^2
        \end{equation*}
        so that if we choose $M_1\in \mathbb{R}$ defined as
        \begin{equation*}
            M_1 = \left(M-N\right)\;\pi \rho
        \end{equation*}
we obtain the required equality (\ref{eq-lipsch}).

        \item $\|p-q\| \leq 2\rho$: in this case
        \begin{align*}
 \Big|\int_{A\smallsetminus B}M \, du \, dv - 
\int_{B\smallsetminus A}N\, du \, dv \Big|  =  
\left(M-N\right) \mu(A\smallsetminus B)\
        \end{align*}
        where $\mu(A\smallsetminus B)$ 
is the area of $A\smallsetminus B$, which is the same as $B\smallsetminus A$.
  
    By looking at the explicit analytic elementary formula for such area, we see
that it is an algebraic function of  $\|p-q\|$, hence we can find $M_2$
satisfying (\ref{eq-lipsch}).
  \end{enumerate}

    Finally, defining
    \begin{equation*}
        L = \max{ \Big\{M_1, M_2\Big\}}
    \end{equation*}
we obtain our result.
\end{proof}

\end{document}